\documentclass{article}

\usepackage{microtype}
\usepackage{graphicx}
\usepackage{subfigure}
\usepackage{booktabs} %

\usepackage{hyperref}

\usepackage[accepted]{icml2019}
\usepackage{colortbl}
\usepackage{xcolor}

\icmltitlerunning{Conservative Objective Models for Offline MBO}

\usepackage{amsmath,amssymb,amsthm}
\usepackage{algorithm}
\usepackage{mathtools}

\newtheorem{theorem}{Theorem}

\newtheorem{proposition}[theorem]{Proposition}

\usepackage{algorithmic}
\usepackage{wrapfig}

\newcommand{\x}{\mathbf{x}}

\newcommand{\data}{\mathcal{D}}

\newcommand{\bx}{\mathbf{x}}

\newcommand{\hatf}{\hat{f}}

\newcommand{\bG}{\mathbf{G}}

\newcommand{\expec}{\mathbb{E}}

\newcommand{\thetastar}{\theta^\star}

\newcommand{\lhat}{\widehat{L}}

\usepackage{amsmath,amsfonts,bm}

\def\eqref#1{equation~\ref{#1}}

\def\1{\bm{1}}

\DeclareMathAlphabet{\mathsfit}{\encodingdefault}{\sfdefault}{m}{sl}
\SetMathAlphabet{\mathsfit}{bold}{\encodingdefault}{\sfdefault}{bx}{n}

\newcommand{\E}{\mathbb{E}}

\begin{document}

\twocolumn[
\icmltitle{Conservative Objective Models for Effective Offline Model-Based Optimization}

\icmlsetsymbol{equal}{*}

\begin{icmlauthorlist}
\icmlauthor{Brandon Trabucco}{equal,berkeley}
\icmlauthor{Aviral Kumar}{equal,berkeley}
\icmlauthor{Xinyang Geng}{berkeley}
\icmlauthor{Sergey Levine}{berkeley}
\end{icmlauthorlist}

\icmlaffiliation{berkeley}{Department of Electrical Engineering and Computer Sciences, University of California Berkeley}

\icmlcorrespondingauthor{Brandon Trabucco}{btrabucco@berkeley.edu}
\icmlcorrespondingauthor{Aviral Kumar}{aviralk@berkeley.edu}

\icmlkeywords{Machine Learning, ICML}

\vskip 0.3in
]

\printAffiliationsAndNotice{\icmlEqualContribution} %

\begin{abstract}
Computational design problems arise in a number of settings, from synthetic biology to computer architectures. In this paper, we aim to solve data-driven model-based optimization (MBO) problems, where the goal is to find a design input that maximizes an unknown objective function provided access to only a static dataset of prior experiments. Such data-driven optimization procedures are the only practical methods in many real-world domains where active data collection is expensive (e.g., when optimizing over proteins) or dangerous (e.g., when optimizing over aircraft designs).
Typical methods for MBO that optimize the design against a learned model suffer from distributional shift: it is easy to find a design that ``fools'' the model into predicting a high value.
To overcome this, we propose \emph{conservative objective models} (COMs), a method that learns a model of the objective function that lower bounds the actual value of the ground-truth objective on out-of-distribution inputs, and uses it for optimization. Structurally, COMs resembles adversarial training methods used to overcome adversarial examples. COMs are simple to implement, and outperform a number existing methods on a wide range of MBO problems, including optimizing protein sequences, robot morphologies, neural network weights, and superconducting materials.

\end{abstract}

\section{Introduction}

\begin{figure}[t]
    \centering
    \includegraphics[width=\linewidth]{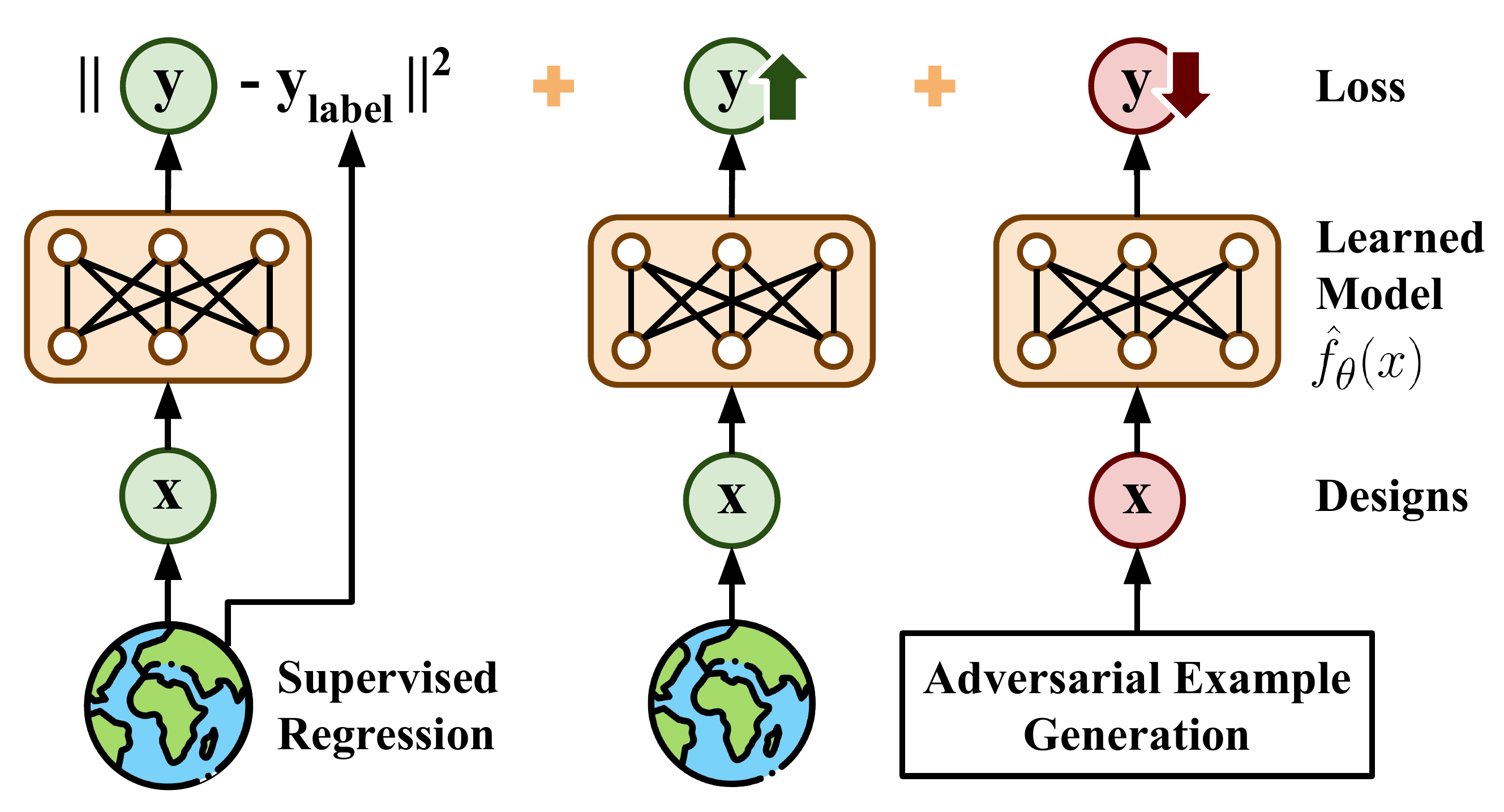}
    \vspace{-12pt}
    \caption{\textbf{Overview of COMs.} Our method trains a model of the objective function by training a neural net with supervised regression on the training data augmented two additional loss terms to obtian conservative predictions. These additional terms aim to maximize the predictions of the neural net model on the training data, and minimize the predictions on adversarially generated designs. This principle prevents the optimizer from producing bad designs with erroneously high values at unseen and poor designs.}
    \label{fig:visual_alg}
    \vspace{-20pt}
\end{figure}

Black-box model-based optimization (MBO) problems are ubiquitous in a wide range of domains, such as protein~\cite{brookes2019conditioning} or molecule design~\cite{gaulton2012chembl}, designing controllers~\cite{berkenkamp2016safe} or robot morphologies~\cite{liao2019}, optimizing neural network designs~\cite{zoph2017}, and aircraft design~\cite{hoburg2012}. Existing methods to solve such model-based optimization problems typically learn a proxy function to represent the unknown objective landscape based on the data, and then optimize the design against this learned objective function. In order to prevent errors in the learned proxy function from affecting optimization, these methods often critically rely on periodic \textit{active} data collection~\cite{snoek2012practical} over the course of training. Active data collection can be expensive or even dangerous: evaluating a real design might involve a complex real-world procedure such as synthesizing candidate protein structures for protein optimization or building the robot for robot design optimization. While these problems can potentially be solved via computer simulation, a high fidelity simulator often requires considerable effort from experts across multiple domains to build, making it impractical for most problems. Therefore, a desirable alternative approach for a broad range of MBO problems is to develop \textit{data-driven, offline} methods that can optimize designs by training highly general and expressive deep neural network models on data from previously conducted experiments, consisting of inputs ($\x$) and their corresponding objective values
($y$), without access to the true function or any form of active data collection~\cite{kumar2019model}. In a number of these practical domains, such as protein~\cite{gfp} or molecule design~\cite{gaulton2012chembl}, plenty of prior data already exists and can be utilized for fully offline, data-driven model-based optimization.

Typical approaches for addressing MBO problems learn a model of the unknown objective function $\hatf$ that maps an input $\x$ (or a representation of the input~\cite{GmezBombarelli2018AutomaticCD}) to its objective value $\hatf(\x)$ via supervised regression on the training dataset~\cite{snoek2012practical}. Then, these methods optimize the input against this learned model via, for instance, gradient ascent.
For MBO problems where the space of valid inputs forms a narrow manifold in a high-dimensional space, any overestimation errors in the learned model will erroneously drive the optimization procedure towards out-of-distribution, invalid, and low-valued design designs that ``fool'' the model into producing a high values~\cite{kumar2019model}.

How can we prevent offline MBO methods from falling into such out-of-distribution solutions? If we can instead learn a \textit{conservative} model of the objective function that does not overestimate the objective value on out-of-distribution inputs, optimizing against this conservative model would produce the best solutions for which we are \emph{confident} in the value.
In this paper, we propose a method to learn such \emph{conservative objective models} (COMs), and then optimize the design against this conservative model using a na\"ive gradient-ascent procedure. 
{Analogously to adversarial training approaches in supervised learning~\citep{goodfellow2014explaining}, and building on recent works in offline reinforcmeent learning~\citep{levine2020offline,kumar2020conservative}, COMs first explicitly mine for out-of-distribution inputs with erroneously overestimated values and then penalize the predictions on these inputs. Theoretically, we show that this approach mitigates overestimation in the learned objective model near the manifold of the dataset. Empirically, we find that this leads to good performance across a range of offline model-based optimization tasks.}

\begin{figure*}[ht!]
    \centering
    \includegraphics[width=0.85\linewidth]{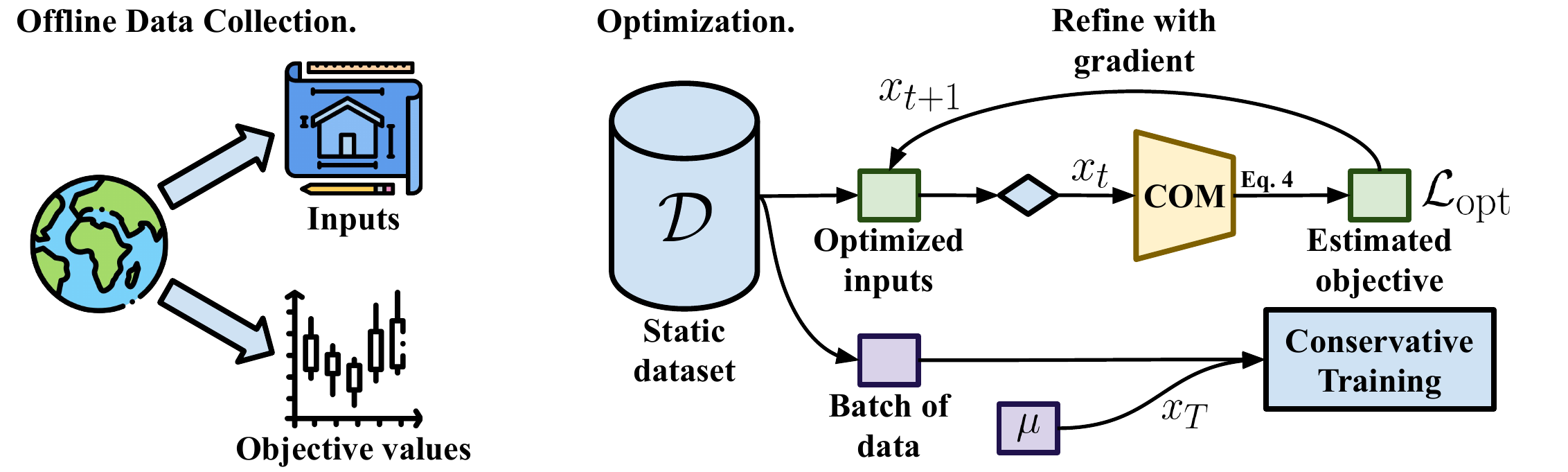}
    \caption{\textbf{Training and optimization using COMs.} The section on the left indicates that each task provides a static dataset that is collected offline without ayn MBO algorithm in-the-loop. The section on the right shows how a conservative objective model is used to produce promising optimized designs using gradient ascent, and how these designs are inputs to a conservative regularizer.}
    \label{fig:coms_diagram}
    \vspace{-10pt}
\end{figure*}

The primary contribution of this paper, COMs, is a novel approach for addressing data-driven model-based optimization problems by learning a conservative model of the unknown objective function that lower-bounds the groundtruth function on out-of-distribution inputs,
and then optimizing the input against this conservative model via a simple gradient-ascent style procedure. COMs are simple to implement, utilizing a supervised learning procedure that resembles adversarial training, without the need for complex generative modeling to estimate dataset support as in prior work on model-based optimization. We theoretically analyze COMs and show that they never overestimate the values at out-of-distribution inputs close to the dataset manifold and we empirically demonstrate the efficacy of COMs on seven complex MBO tasks that span a wide range of real-world tasks including biological sequence design, neural network parameter optimization, and superconducting material design. COMs is optimal on \textbf{4/7} tasks, and outperforms the best prior method by a factor of \textbf{1.3x} in a high-dimensional setting, and by a factor of \textbf{1.16x} overall.

\section{Preliminaries}
\label{sec:prelims}
\vspace{-5pt}

The goal in data-driven, offline model-based optimization~\cite{kumar2019model} is to find the best possible solution, $\x^*$, to optimization problems of the form
\begin{equation}
\label{eqn:forward_model}
    \x^* \leftarrow \arg \max_{\x}~ f(\x),
\end{equation}
where $f(\x)$ is an unknown (possibly stochastic) objective function. An offline MBO algorithm is provided access to a static dataset $\mathcal{D}$ of inputs and their objective values, $\mathcal{D} = \{(\mathbf{x}_1, y_1), \cdots, (\mathbf{x}_N, y_N) \}$.
While a variety of MBO methods have been developed~\cite{GmezBombarelli2018AutomaticCD,brookes2019conditioning,kumar2019model,fannjiang2020autofocused}, most methods for tackling MBO problems fit a parametric model to the samples of the true objective function in $\mathcal{D}$, $\hatf_\theta(\mathbf{x})$, via supervised training: $\hat{f}_\theta(\mathbf{x}) \leftarrow \arg \min_{\theta} \sum_{i} (\hatf_\theta(\mathbf{x}_i) - y_i)^2$, and find $\x^*$ in Equation~\ref{eqn:forward_model} by optimizing $\x$ against this learned model $\hat{f}_\theta(\mathbf{x})$, typically with some mechanism to additionally minimize distribution shift. One choice for optimizing $\x$ in Equation~\ref{eqn:forward_model} is gradient descent on the learned function, as given by
\begin{equation}
\label{eqn:gradient_ascemt}
\x_{k+1} \leftarrow \x_{k} + \eta \nabla_x \hatf_\theta(\x)\vert_{x = \x_k}, ~~~ \text{for~~} k \in [1, T], ~~~~\x^\star = \x_T.    
\end{equation}
The fixed point of the above procedure $\x_T$ is then the output of the MBO procedure. In high-dimensional input spaces, where valid $\x$ values lie on a thin manifold in a high-dimensional space, such an optimization procedure is prone to producing low-scoring inputs, which may not even be valid. This is because $\hatf$ may erroneously overestimate objective
values at out-of-distribution points,
which would naturally lead the optimization to such invalid points. Prior methods have sought to address this issue via generative modeling or explicit density estimation, so as to avoid out-of-distribution inputs. In the next section, we will describe how our method, COMs, instead trains the objective model in such a way that overestimation is prevented directly.

\vspace{-5pt}
\section{Conservative Objective Models for Offline Model-Based Optimization}
\label{sec:method}
\vspace{-5pt}
In this section, we present our approach, conservative objective models (COMs). COMs learn estimates of the true function that do not overestimate the value of the ground truth objective on out-of distribution inputs in the vicinity of the training dataset.
As a result, COMs prevent erroneous overestimation that would drive the optimizer (Equation~\ref{eqn:gradient_ascemt}) to produce out-of-distribution inputs with low values under the groundtruth objective function. We first discuss a procedure for learning such conservative estimates and then explain how these conservative models can be used for offline MBO.

\subsection{Learning Conservative Objective Models (COMs)}
The key idea behind our approach is to augment the objective for training of the objective model, $\hatf_\theta(\x)$, with a regularizer that minimizes the expected value of this function on {``adversarial'' inputs where the value of the learned function $\hatf_\theta$ may be erroneously large. Such adversarial inputs are likely to be found by the optimizer during optimization, and hence, we need to train the learned function to not overestimate their values. How can we compute such adversarial inputs? Building on simple techniques for generating adversarial examples in supervised learning~\citep{goodfellow2014explaining}, we can run multiple steps of gradient ascent on the current snapshot of the learned function $\hatf(\bx)$ starting from various inputs in the training dataset to obtain such adversarial inputs. For concise notation in the exposition, we denote the distribution of all adversarial inputs found via this gradient ascent procedure as $\mu(\x)$. Samples from $\mu(\x)$ are obtained by sampling a datapoint from the training set and running several steps of gradient ascent on $\hatf(\bx)$.} 
\begin{align}
\label{eqn:csm_mu}
    \mu(\x) = \sum_{\x_0 \in \data} \delta_{\x = \x_T}: \x_{t+1} = \x_t + \eta \nabla_\x \hatf_\theta(\x)\big\vert_{\x = \x_t}
\end{align}

{While simply minimizing the function values under this adversarial distribution $\mu(\x)$ should effectively reduce the value of the learned $\hatf$ at these inputs, this can result in systematic underestimation even for in-distribution points. To balance out this regluarization,
our approach additionally \emph{maximizes} the expected value of this function on the training dataset. This can be formalized as maximizing the value of $\hatf(\bx)$ under the empirical distribution of inputs $\bx \in \mathcal{D}$ given by:  $\hat{D}(\x) = \sum_{\x_i \in \mathcal{D}} \delta_{\x = \x_i}$.}
In Section~\ref{sec:theory}, we will show that the minimization and maximization terms balance out, and this objective learns a function $\hatf_\theta(\x)$ that is a lower bound on the true function $f(\x)$ for inputs that are encountered during the optimization process, under several assumptions. This approach is inspired by recent work in offline RL~\citep{kumar2020conservative}, where a similar objective is used to learn conservative value functions. We will elaborate on this connection in Section~\ref{sec:related}.
Formally, our
training objective is given by the following equation, where $\alpha$ is a parameter that trades off conservatism for regression:
\begin{align}
\label{eqn:csm_training}
    \hatf_\thetastar \leftarrow \arg \min_{\theta \in \Theta}~~ & \overbrace{\alpha \left( \expec_{\x \sim \mu(\x)}\left[\hatf_\theta(\x)\right] - \expec_{\x \sim \mathcal{D}}\left[\hatf_\theta(\x)\right] \right)}^{\text{COMs regularizer}} \nonumber\\& + \frac{1}{2} \underbrace{\mathbb{E}_{(\x, y) \sim \mathcal{D}}\left[ \left(\hatf_\theta(\x) - y \right)^2 \right]}_{\text{standard supervised regression}},
\end{align}

This idea is schematically depicted in Figure~\ref{fig:visual_alg}. The value of $\alpha$ and the choice of distribution $\mu(\x)$ play a crucial role in determining the behavior of this approach. 
If the chosen $\alpha$ is very small, then the resulting $\hatf_\thetastar(\x)$ may not be a conservative estimate of the actual function $f(\x)$, whereas if the chosen $\alpha$ is too large, then the learned function will be too conservative, and not allow the optimizer to deviate away from the dataset at all. We will discuss our strategy for choosing $\alpha$ in the next section. As noted earlier, our choice of $\mu(\x)$ specifically focuses on adversarial inputs that the optimizer is likely to encounter while optimizing the input. We compute this distribution $\mu(\x)$ by sampling a starting point $\x_0$ from the dataset $\data$, and then performing several steps of gradient ascent on $\hatf_\theta$ starting from this point.

\begin{figure*}
\begin{minipage}{0.53\textwidth}
\vspace{-10pt}
\begin{algorithm}[H]
\small
\centering
\caption{COM: Training Conservative Models}
\begin{algorithmic}[1]
\STATE Initialize $\hatf_\theta$. Pick $\eta, \alpha$ and initialize dataset $\data$.
    \FOR{$i = 1$ to training\_steps}
        \STATE Sample $(\x_0, y) \sim \mathcal{D}$
        \STATE Find $\x_T(\x_0)$ via gradient ascent from $\x_0$:\\ ~~ $\x_{t+1} = \x_t + \eta \nabla_\x \hatf_\theta(\x)\big\vert_{\x = \x_t}$; ~~$\mu(\x) = \sum_{\x_0 \in \data} \delta_{\x = \x_T(\x_0)}$.
        \STATE Minimize $\mathcal{L}(\theta; 
        \alpha)$ with respect to $\theta$.\\ $\mathcal{L}(\theta; \alpha)\!=\!\E_{\x_0 \sim \data}( \hatf_{\theta}(\x_0)\!-\!y )^{2}\!-\!\alpha \E_{\x_0}[\hatf_{\theta}(\x_0)] + \alpha \E_{\mu(\x)}[\hatf_{\theta}(\x)]$ \\
        $\theta \leftarrow \theta - \lambda \nabla_{\theta} \mathcal{L}(\theta; 
        \alpha)$ 
    \ENDFOR
\end{algorithmic}
\label{alg:training}
\vspace{1pt}
\end{algorithm}
\end{minipage}
~\vline~
\begin{minipage}{0.44\textwidth}
\vspace{-10pt}
\begin{algorithm}[H]
\centering
\small
\caption{COM: Finding $\x^\star$}
\begin{algorithmic}[1]
\STATE Initialize optimizer at the optimum in $\mathcal{D}$:\\ ~~~~~ $\tilde \x = \arg\max_{(\x, y) \in \mathcal{D}} ~y$
\STATE Find $\x^\star$ via gradient ascent from $\tilde \x$:\\ ~~~~~$\x_{t+1} = \x_t + \eta \nabla_\x \mathcal{L}_{\text{opt}}(\x)\big\vert_{\x = \x_t}$\\ 
~~~~$\text{where}~~ \mathcal{L}_{\text{opt}}(\x) := \hatf_\thetastar(\x)$
\STATE Return the solution $\x^\star = \x_{T}$.
\end{algorithmic}
\label{alg:evaluation}
\end{algorithm}
\end{minipage}
\end{figure*}

\subsection{Optimizing a Conservative Objective Model}
Once we have a trained conservative model from Equation~\ref{eqn:csm_training}, we must use this learned model for finding the best possible input, $\x^\star$. Prior works~\cite{kumar2019model,brookes2019conditioning} use a standard (non-conservative) model of the objective function in conjunction with generative models or density estimators to restrict the optimization to in-distribution values of $\x^\star$. However, since our conservative training method trains $\hatf_\thetastar$ to explicitly assign low values to out-of-distribution inputs, we can use a simple gradient-ascent style procedure in the input space to find the best possible solution.

Specifically, our optimizer runs gradient-ascent for $T$ iterations starting from an input in the dataset ($\x_0 \in \data$), in each iteration trying to move the design in the direction of the gradient of the learned model $\hatf_\thetastar$. 
Starting from the best point in the dataset, $\x_0 \in \mathcal{D}$, our optimizer performs the following update (also shown in Algorithm~\ref{alg:evaluation}, Line 2):
\begin{align}
    &\forall~ t \in [T], \x_0 \in \mathcal{D};~~~~ \x_{t+1} = \x_t + \eta \nabla_\x \mathcal{L}_{\text{opt}}(\x)\big\vert_{\x = \x_t} \nonumber\\ 
    &\text{where}~~ \mathcal{L}_{\text{opt}}(\x) := \hatf_\thetastar(\x). 
\label{eqn:csm_optimization_constraint}
\end{align}
Equation~\ref{eqn:csm_optimization_constraint} ensures that the value of the learned function $\hatf_\theta(\x_{t+1})$ is larger than the value at its previous iterate $\x_t$. Furthermore, the number of iterations $T$ of gradient ascent during optimization in Equation~\ref{eqn:csm_optimization_constraint} is the identical to the number of steps that we use to generate adversarial examples, $\mu(\x)$ in Equation~\ref{eqn:csm_mu}. This ensures that the optimizer only queries the region when the learned function $\hatf_\theta(\x)$ is indeed conservative and a valid lower bound.

\subsection{Using COMs for MBO: Additional Decisions}
\label{sec:method_addl}

Next we discuss other design decisions that appear in COMs training (Equation~\ref{eqn:csm_training}) or when optimizing the input against a learned conservative model (Equation~\ref{eqn:csm_optimization_constraint}).

\underline{\textbf{Choosing $\alpha$.}} The hyperparameter $\alpha$ in Equation~\ref{eqn:csm_training} plays an important role in weighting conservatism against accuracy. Without access to additional active data collection for evaluation, tuning this hyperparameter for each task can be challenging. Therefore, in order to turn COMs into a task-agnostic algorithm for offline MBO, we devise an automated procedure for selecting $\alpha$. As discussed previously, if $\alpha$ is too large, $\hatf_\thetastar$ is expected to be too conservative, since it would assign higher values to points in the dataset, and low values to \emph{all} other points. 
Selecting a single value of $\alpha$ that works for many problems is difficult, since its effect depends strongly on the magnitude of the objective function. Instead, we use a modified training procedure that poses Equation~\ref{eqn:csm_training} as a constrained optimization problem, with $\alpha$ assuming the role of a Lagrange dual variable for satisfying a constraint that controls the difference in values of the learned objective under $\mu(\x)$ and $\mathcal{D}(\x)$. This corresponds to solving the following optimization problem:  
\begin{multline}
\label{eqn:csm_training_lagrange}
    \hatf_\thetastar \leftarrow \arg \min_{\theta \in \Theta}~~ \frac{1}{2} \mathbb{E}_{(\x, y) \sim \mathcal{D}}\left[ \left(\hatf_\theta(\x) - y \right)^2 \right]\\
    ~~ \text{s.t.}~~ \left( \expec_{\x \sim \mu(\x)}\left[\hatf_\theta(\x)\right] - \expec_{\x \sim \mathcal{D}}\left[\hatf_\theta(\x)\right] \right) \leq \tau.
\end{multline}
While Equation~\ref{eqn:csm_training_lagrange} introduces a new hyperparameter $\tau$ in place of $\alpha$, this parameter is easier to select by hand, since its optimal value does not depend on the magnitude of the objective function as we can normalize the objective values to the same range before use in Equation~\ref{eqn:csm_training_lagrange}, and therefore, a single choice works well across a diverse range of tasks. We find that a single value of $\tau$ is effective on every continuous task ($\tau=0.5$) and discrete task ($\tau=2.0$) respectively, and empirically ablate the choice of $\tau$ in Figure~\ref{fig:gradient_ascent}.

\underline{\textbf{Selecting optimized designs $\x^\star$.}} So far we have discussed how COMs can be trained and used for optimization; however, we have not established a way to determine which $\x_t$ (Equation~\ref{eqn:csm_optimization_constraint}) encountered in the optimization trajectory should be used as our final solution $\x^\star$. The most natural choice is to pick the final $\x_T$ found by the optimizer as the solution. We uniformly choose $T=50$ steps. While the choice of $T$ should, in principle, affect the solution found by any gradient-ascent style optimizer, we found COMs to be quite stable to different values of $T$, as we will elaborate empirically on in Section~\ref{sec:ablations}, Figure~\ref{fig:gradient_ascent}. Of course, there are many other possible ways of selecting $T$, including ideas inspired from offline model-selection methods in offline reinforcement learning~\citep{thomas2015high}, but our simple procedure, which is also popular in offline RL~\citep{fu2020d4rl}, ensures that the optimizer only queries the regions of the input space where the learned function is indeed trained to be conservative and is also sufficient to obtain good optimization performance.

\subsection{Overall Algorithm and Practical Implementation} 
\label{sec:method_overall}
Finally, we combine the individual components discussed so far to obtain a complete algorithm for offline model-based optimization. Pseudocode for our algorithm is shown in Algorithm~\ref{alg:training}. COMs parameterize the objective model, $\hatf_\theta(\x)$, via a feed-forward neural network with parameters $\theta$. Our method then alternates between approximately generating samples $\mu(\x)$ via gradient ascent (Line 4), and optimizing parameters $\theta$ using Equation~\ref{eqn:csm_optimization_constraint} (Line 5). 
Finally, at the end of training, we run the gradient ascent procedure over the learned objective model $\hatf_\thetastar(\x)$ for a large $T$ number of ascent steps and return the final design $\x_T$ as $\x^\star$.

\underline{\textbf{Implementation details.}} Full implementation details for our method can be found in Appendix~\ref{app:method_details}. Briefly, for all of our experiments, the conservative objective model $\hatf_\theta$ is modeled as a neural network with two hidden layers of size 2048 each and leaky ReLU activations. 
More details on the network structure can be found in Appendix~\ref{appendix:network}. In order to train this conservative objective model, we use the Adam optimizer~\cite{kingma2014adam} with a learning rate of $10^{-3}$. 
Empirically, we found that if $\eta$ is too large, gradient ascent begins to produce inputs $\x_T$ that do not maximize the values of $\hatf_\thetastar(\x_T)$, so we select the largest $\eta$ such that successive $\x_t$ follow the gradient vector field of $\hatf_\thetastar(\x_t)$. For computing samples $\mu(\x)$, we used 50 gradient ascent steps starting from a given design in the dataset, $\x_0 \in \data$. 
During optimization, we used the gradient-ascent optimizer with a learning rate of 0.05 for continuous tasks and 2.0 for discrete tasks. As we will also show in our experiments (Section~\ref{sec:ablations}), this produces stable optimization behavior for all tasks we attempted. Finally, in order to choose the step $T$ in Equation~\ref{eqn:csm_optimization_constraint}  that is supposed to provide us with the final solution $\x^\star = \x_T$, we pick a universal step of $T=50$.

\section{Theoretical Analysis of COMs}
\label{sec:theory}
We will now theoretically analyze conservative objective models, and show that the conservative training procedure (Equation~\ref{eqn:csm_training}) indeed learns a conservative model of the objective function. To do so, we will show that under Equation~\ref{eqn:csm_training}, the values of all inputs in regions found within $T$ steps of gradient ascent starting from any input $\x_0 \in \mathcal{D}$ are lower-bounds on their actual value. For analysis, we will denote $\overline{\data}(\x)$ as the smoothed density of $\x$ in the dataset $\data$ (see Appendix~\ref{appendix:theorem_proof} for a formal definition).
We will express Equation~\ref{eqn:csm_training} in an equivalent form that factorizes the distribution $\mu(\x)$ as $\mu(\x) = \sum_{\x_0 \sim \data} \overline{\data}(\x_0) \mu(\x_T|\x_0)$:
\begin{multline}
\label{eqn:csm_analysis}
\min_{\theta \in \Theta}~ \alpha \left( \expec_{\textcolor{red}{\x_0 \sim \overline{\mathcal{D}}, \x_T \sim \mu(\x_T|\x_0)}}\left[\hatf_\theta(\x_T)\right] - \expec_{\x \sim \overline{\mathcal{D}}}\left[\hatf_\theta(\x)\right] \right) \\ + \frac{1}{2} \mathbb{E}_{(\x, y) \sim \overline{\mathcal{D}}}\left[ \left(\hatf_\theta(\x) - y \right)^2 \right].
\end{multline}
While $\mu(\x_T|\x_0)$ is a Dirac-delta distribution in practice (Section~\ref{sec:method}), for our analysis, we will assume that it is a distribution centered at $\x_T$ and $\mu(\x_T|\x_0) > 0~ \forall~ \x_T \in \mathcal{X}$. This condition can be easily satisfied by adding random noise during gradient ascent while computing $\x_T$. We will train $\hat{f}_\theta$ using gradient descent and denote $k=1,2,\cdots$ as the iterations of this training procedure for $\hat{f}_\theta$.

We first summarize some assumptions used in our analysis. We assume that the true function $f(\x)$ is $L$-Lipschitz over the input space $\x$. We also assume that the learned function $\hat{f}_\theta(\x)$ is $\lhat$-Lipschitz and $\lhat$ is sufficiently larger than $L$. 
For analysis purposes, we will define a conditional distribution, $\overline{\data}(\x'|\x)$, to be a Gaussian distribution centered at $\x$: $\mathcal{N}(\x'| \x, \sigma^2)$.
We will not assume a specific parameterization for the objective model, $\hat{f}_\theta$, but operate under the neural tangent kernel (NTK)~\citep{jacot2018neural} model of neural nets. The neural tangent kernel of the function $\hat{f}(\x)$ be defined as: $\bG_f(\x_i, \x_j) := \nabla_\theta \hat{f}_\theta(\x_i)^T \nabla_\theta \hat{f}_\theta(\x_j)$. Under these assumptions, we build on the analysis of conservative Q-learning~\citep{kumar2020conservative} to prove our theoretical result in Theorem~\ref{thm:lower_bound}, shown below:

\begin{proposition}[Conservative training lower-bounds the true function]
\label{thm:lower_bound}
Assume that $\hat{f}_\theta(\x)$ is trained with conservative training by performing gradient descent on $\theta$ with respect to the objective in Equation~\ref{eqn:csm_analysis} with a learning rate $\eta$. The parameters in step $k$ of gradient descent are denoted by $\theta^k$, and let the corresponding conservative model be denoted as $\hat{f}^k_\theta$. Let $\bG$, $\mu$, $\lhat$, $L$, $\overline{\data}$ be defined as discussed above. Then, under assumptions listed above, $\forall~ \x \in \data, \x'' \in \mathcal{X}$, the conservative model at iteration $k+1$ of training satisfies:
\begin{align*}
    \hat{f}_\theta^{k+1}(\x'') := \max~ &\Big\{\hat{{f}}^{k+1}_\theta(\x) - \lhat ||\x'' - \x||_2,\\
    &\tilde{f}_\theta^{k+1}(\x'') - \eta \alpha \E_{\x \sim \overline{\data}, \x' \sim \mu}[\bG^k_f(\x'', \x')]\\ &+ \eta \alpha \E_{\x \sim \overline{\data}, \x' \sim \overline{\data}} [\bG^k_f(\x'', \x')]\Big\},
\end{align*}
where $\tilde{f}_\theta^{k+1}(\x'')$ is the resulting $(k\!+\!1)$-th iterate of $\hat{f}_\theta$ if conservative training were not used. Thus, if $\alpha$ is sufficiently large, the expected value of the asymptotic function, $\hat{f}_\theta := \lim_{k \rightarrow \infty} \hat{f}_\theta^k$, on inputs $\x_T$ found by the optimizer, lower-bounds the value of the true function $f(\x_T)$: $$\E_{\x_0 \sim \data, \x_T \sim \mu(\x_T|\x_0)}[\hat{f}_\theta(\x_T)] \leq \E_{\x_0 \sim \data, \x_T \sim \mu(\x_T|\x_0)}[f(\x)].$$ 
\end{proposition}

A proof for Proposition~\ref{thm:lower_bound} including a complete formal statement can be found in Appendix~\ref{appendix:theorem_proof}. The intuition behind the proof is that inducing conservatism in the function $\hat{f}_\theta$ at each gradient step of optimizing Equation~\ref{eqn:csm_analysis} makes the asymptotic function be conservative. Moreover, the larger the value of $\alpha$, the more conservative the function $\hat{f}_\theta$ is on points $\x'$ found via gradient ascent, i.e., points with high density under $\mu(\x_T|\x_0)$, in expectation. Finally, when gradient ascent is used to find $\x^\star$ on the learned conservative model, $\hat{f}_\theta$, and the number of steps of gradient ascent steps is less than $T$, as we do in practice via Equation~\ref{eqn:csm_optimization_constraint}, this bound with additional offset will hold for the point $\x^\star$ in expectation, and therefore the estimated value of this point will not overestimate its true value. 
This additional offset depends on the Lipschitz constant $\lhat$ and the distance between $\x^*$ and the the optimized solutions $\x_T$ found for other data points, $\x_0 \in \mathcal{D}$.

\begin{table*}[t]
    \centering
    \setlength\tabcolsep{4pt}
    \begin{tabular}{l||r|r|r|r|>{\columncolor[gray]{0.9}}r}
    \toprule
    {}            &                        \textbf{GFP} &         \textbf{TF Bind 8} &                        \textbf{UTR} & \textbf{$\#$ Optimal} & \textbf{Norm. avg. perf.}  \\
    \midrule
    $\mathcal{D}$ (\textbf{best}) & 0.789 & 0.439 & 0.593 & & \\
    Auto. CbAS    &          \textbf{0.865 $\pm$ 0.000} & 0.910 $\pm$ 0.044 &          0.691 $\pm$ 0.012 &                      1 / 7 &  0.687 \\
    CbAS          &          \textbf{0.865 $\pm$ 0.000} & 0.927 $\pm$ 0.051 &          \textbf{0.694 $\pm$ 0.010} &                   3 / 7 &    0.699\\
    BO-qEI        &          0.254 $\pm$ 0.352 & 0.798 $\pm$ 0.083 &          0.684 $\pm$ 0.000 &                     0 / 7 &  0.629  \\
    CMA-ES        &          0.054 $\pm$ 0.002 & 0.953 $\pm$ 0.022 &          \textbf{0.707 $\pm$ 0.014} &             2 / 7 & 0.674  \\
    Grad.         &          0.864 $\pm$ 0.001 & \textbf{0.977 $\pm$ 0.025} &          \textbf{0.695 $\pm$ 0.013} &             3 / 7 & 0.750   \\
    Grad. Min     &          0.864 $\pm$ 0.000 & \textbf{0.984 $\pm$ 0.012} &          \textbf{0.696 $\pm$ 0.009} &             3 / 7  & 0.829      \\
    Grad. Mean     &          0.864 $\pm$ 0.000 & \textbf{0.986 $\pm$ 0.012} &          0.693 $\pm$ 0.010 &            2 / 7 & 0.852           \\
    MINs          &          \textbf{0.865 $\pm$ 0.001} & 0.905 $\pm$ 0.052 &          \textbf{0.697 $\pm$ 0.010} &            \textbf{4 / 7} &   0.745          \\
    REINFORCE     &          \textbf{0.865 $\pm$ 0.000} & 0.948 $\pm$ 0.028 &          0.688 $\pm$ 0.010 &             1 / 7 & 0.541    \\
    \midrule
    \textbf{COMs (Ours)} & \textit{0.864 $\pm$ 0.000} & \textit{0.945 $\pm$ 0.033} & \textbf{0.699 $\pm$ 0.011} &    \textbf{4 / 7} & \textbf{0.985}  \\
    \midrule
    \midrule
    {}            &    \textbf{Superconductor} &     \textbf{Ant Morphology} &     \textbf{D'Kitty Morphology} &     \textbf{Hopper Controller} \\
    \midrule
    $\mathcal{D}$ (\textbf{best}) & 0.399 & 0.565 & 0.884 & 1.0 \\
    Auto. CbAS    & 0.421 $\pm$ 0.045 &          0.882 $\pm$ 0.045 &          0.906 $\pm$ 0.006 &          0.137 $\pm$ 0.005 \\
    CbAS          & \textbf{0.503 $\pm$ 0.069} &          0.876 $\pm$ 0.031 &          0.892 $\pm$ 0.008 &          0.141 $\pm$ 0.012 \\
    BO-qEI        & 0.402 $\pm$ 0.034 &          0.819 $\pm$ 0.000 &          0.896 $\pm$ 0.000 &          0.550 $\pm$ 0.118 \\
    CMA-ES        & 0.465 $\pm$ 0.024 &          \textbf{1.214 $\pm$ 0.732} &          0.724 $\pm$ 0.001 &          0.604 $\pm$ 0.215 \\
    Grad.         & \textbf{0.518 $\pm$ 0.024} &          0.293 $\pm$ 0.023 &          0.874 $\pm$ 0.022 &          1.035 $\pm$ 0.482 \\
    Grad. Min     & \textbf{0.506 $\pm$ 0.009} &          0.479 $\pm$ 0.064 &          0.889 $\pm$ 0.011 &          1.391 $\pm$ 0.589 \\
    Grad. Mean     & \textbf{0.499 $\pm$ 0.017} &          0.445 $\pm$ 0.080 &          0.892 $\pm$ 0.011 &          1.586 $\pm$ 0.454 \\
    MINs          & 0.469 $\pm$ 0.023 &          \textbf{0.913 $\pm$ 0.036} &          \textbf{0.945 $\pm$ 0.012} &          0.424 $\pm$ 0.166 \\
    REINFORCE     & 0.481 $\pm$ 0.013 &          0.266 $\pm$ 0.032 &          0.562 $\pm$ 0.196 &         -0.020 $\pm$ 0.067 \\
    \midrule
    \textbf{COMs (Ours)} & 0.439 $\pm$ 0.033 & \textbf{0.944 $\pm$ 0.016} & \textbf{0.949 $\pm$ 0.015} & \textbf{2.056 $\pm$ 0.314} \\
    \bottomrule
    \end{tabular}
    \caption{\textbf{Comparative evaluation of COMs} against prior methods in terms of the mean 100th-percentile score and its standard deviation over 8 trials. Tasks include \textbf{Superconductor-RandomForest-v0}, \textbf{HopperController-Exact-v0}, \textbf{AntMorphology-Exact-v0}, and \textbf{DKittyMorphology-Exact-v0}, which have a continuous design space and \textbf{GFP-Transformer-v0}, \textbf{TFBind8-Exact-v0}, and \textbf{UTR-ResNet-v0} with a discrete design space. COMs perform strictly better on high-dimensional tasks, obtaining about \textbf{1.3x} gains on Hopper Controller, and compelling gains on Ant Morphology and D'Kitty Morphology tasks. In addition, COMs is able to consistently find solutions that outperform the best training point for each task, given by \textbf{$\mathcal{D}$ (best)}. For each task, algorithms within one standard deviation of having the highest performance are bolded. COMs attain the optimal performance in 4/7 tasks (``\# Optimal'') attaining a normalized average performance of \textbf{0.985} compared to {0.852} for the next best method, outperforming other methods as indicated.} 
    \label{tab:perf}
    \vspace{-.3cm}
\end{table*}

\section{Related Work}
\label{sec:related}

We now briefly discuss prior works in MBO, including prior work on active model-based optimization and work that utilizes offline datasets for data-driven MBO.

\textbf{Bayesian optimization.} Most prior work on model-based optimization has focused on the active setting, where derivative free methods such as the cross-entropy method~\cite{rubinstein2004cross} and other methods derived from the REINFORCE trick~\cite{williams,rubinstein96optimizationof}, reward-weighted regression~\cite{peters2012reinforcement}, and Gaussian processes~\cite{snoek15scalable,shahriari2016TakingTH,snoek2012practical} have been utilized. Most of these methods focus mainly on low-dimensional tasks with active data collection. Practical approaches have combined these methods with Bayesian neural networks~\cite{snoek15scalable,snoek2012practical}, latent variable models~\cite{kim2018attentive,garnelo18neural,garnelo18conditional}, and ensembles of learned score models~\cite{angermueller2020population,Angermueller2020Model-based,mirhoseini2020chip}. These methods still require actively querying the true function $f(\x)$. Further, as shown by \citep{brookes2019conditioning,fannjiang2020autofocused,kumar2019model}, these Bayesian optimization methods are susceptible to producing invalid out-of-distribution inputs in the offline setting. Unlike these methods, COMs are specifically designed for the offline setting with high-dimensional inputs, and avoid out-of-distribution inputs. 

\textbf{Offline model-based optimization.} Recent works have also focused on optimization in the completely offline setting. Typically these methods utilize a generative model~\citep{kingma2013autoencoding,goodfellow2014generative} that models the manifold of inputs. \citep{brookes2019conditioning,fannjiang2020autofocused} use a variational autoencoder~\cite{kingma2013autoencoding} to model the space of $\x$ and use it alongside a learned objective function. \citep{kumar2019model} use a generative model to parameterize an inverse map from
the scalar objective $y$ to input $\x$ and search for the optimal one-dimensional $y$ during optimization.
Modeling the manifold of valid inputs globally can be extremely challenging (see Ant, Hopper, and DKitty results in Section~\ref{sec:exps}), and as a result these generative models often need to be tuned for each domain~\citep{trabucco2021designbench}. In contrast, COMs do not require any generative model, and fit an approximate objective function with a simple regularizer, providing both a simpler, easier-to-use algorithm and better empirical performance. \citet{Fu21nemo} also avoid training a generative model, but instead use normalized maximum likelihood, which requires training multiple discriminative models---COMs only requires one---and quantizing $y$, which COMs does not.

\textbf{Adversarial examples.} As discussed in Section~\ref{sec:prelims}, MBO methods based on learned objective models naturally query the learned function on ``adversarial'' inputs, where the learned function erroneously overestimates the true function. This is superficially similar to adversarial examples in supervised learning~\cite{goodfellow2014explaining}, which can be generated by maximizing the input against the loss function. While adversarial examples have been formalized as out-of-distribution inputs lying in the vicinity of the data distribution and prior works have attempted to correct for them by encouraging smoothness~\cite{Florian2018} of the learned function, and there is evidence that robust objective models help mitigate over estimation~\cite{Santurkar19RobustClassifier}, these solutions may be ineffective in MBO settings when the true function is itself non-smooth. Instead making conservative predictions on such adversarially generated inputs may prevent poor performance.

\section{Experimental Evaluation}
\label{sec:exps}
To evaluate the efficacy of COMs for offline model-based optimization, we first perform a comparative evaluation of COMs on four continuous and three discrete offline MBO tasks based on problems in physical sciences, neural network design, material design, and robotics, proposed in the design-bench benchmark~\cite{trabucco2021designbench}, that we also describe shortly.
In addition, we perform an empirical analysis on COMs that aims to answer the following questions: {\textbf{(1)} Is conservative training essential for improved performance and stability of COMs? How do COMs compare to a na\"ive objective model in terms of stability?, \textbf{(2)} {How sensitive are COMs are to various design choices during optimization?}, \textbf{(3)} Are COMs robust to hyperparameter choices and consistent to evaluation conditions?} We answer these questions by studying the behavior of COMs under controlled conditions and using visualizations for our analysis. Code for reproducing our results is at \href{https://github.com/brandontrabucco/design-baselines/blob/c65a53fe1e6567b740f0adf60c5db9921c1f2330/design_baselines/coms_cleaned/__init__.py}{https://github.com/brandontrabucco/design-baselines}

\begin{figure}[t]
    \centering
    \includegraphics[trim=50 5 70 5,clip,width=\linewidth]{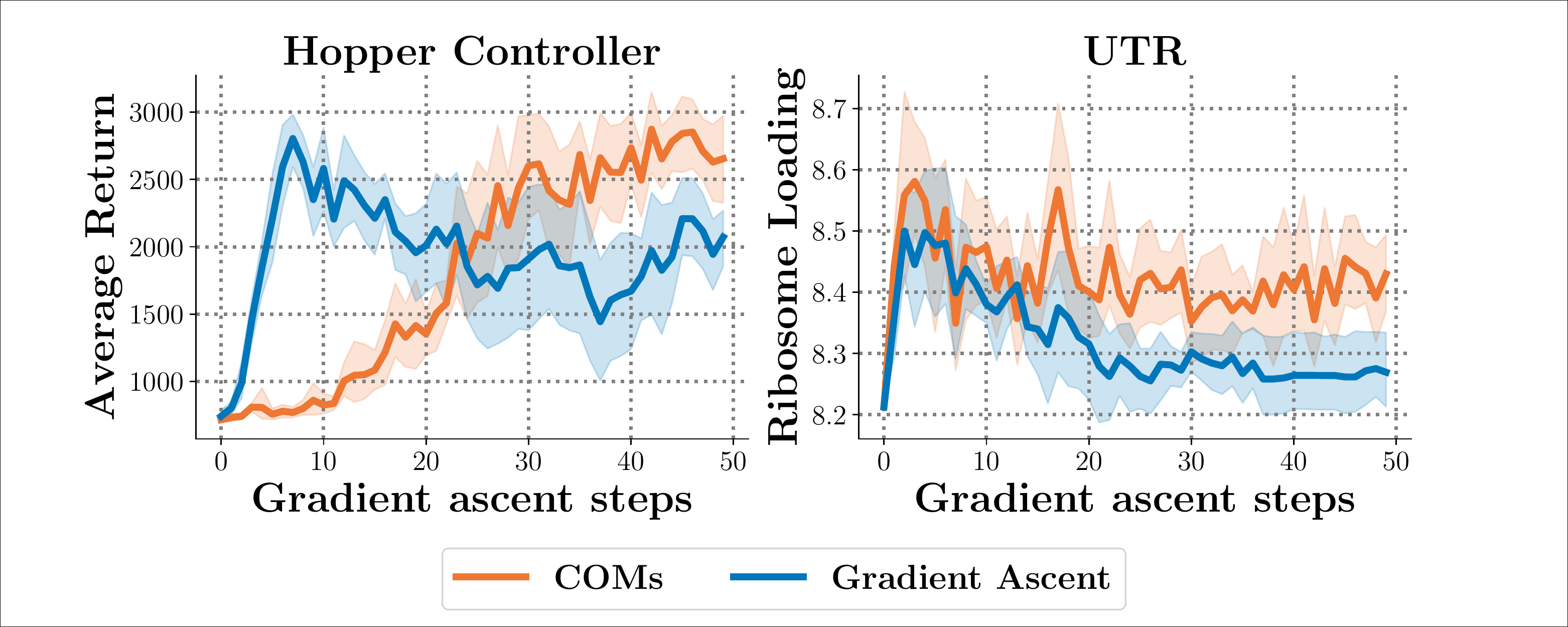}
    \vspace{-0.75cm}
    \caption{\textbf{Stability of COMs versus na\"{i}ve gradient ascent.} The x-axis shows the number of gradient ascent steps taken on the design $\x^*$, and the y-axis shows the 100th percentile of the ground truth task objective function evaluated at every gradient step, which is used only for analysis only and is unavailable to the algorithm. In both cases, COMs reach solutions that remain at higher performance stably, indicating that COMs are less sensitive to varying numbers of gradient ascent steps performed during optimization.}
    \label{fig:gradient_ascent}
    \vspace{-.5cm}
\end{figure}

\subsection{Empirical Performance on Benchmark Tasks}
\label{sec:results_benchmarks}

We first compare COMs to a range of recently proposed methods for offline MBO in high-dimensional input spaces: CbAS~\cite{brookes2019conditioning}, MINs~\cite{kumar2019model} and and autofocused CbAS~\cite{fannjiang2020autofocused}, that augments CbAS with a re-weighted objective model. Additionally, we also compare COMs to more standard baseline algorithms including REINFORCE~\cite{Williams92}, CMA-ES~\cite{Hansen06} , and BO-qEI, Bayesian Optimization with the quasi-expected improvement acquisition function \citep{reparameterization2017}. We also compare to a na\"ive \textbf{gradient ascent} baseline that first learns a model of the actual function via supervised regression (with no conservative term like COMs) and then optimizes this learned proxy via gradient ascent.  CbAS variants and MINs train generative models such as VAEs~\cite{kingma2013autoencoding} and GANs~\cite{goodfellow2014generative}, which generally require task-specific neural net architectures, as compared to the substantially simpler discriminative models used for COMs. In fact, we use the same architecture for COMs on all the tasks. In addition, we instantiate this gradient ascent baseline with an ensemble of learned models of the objective function, with either a minimum (Grad. Min.) or mean (Grad. Mean) over the ensemble to obtain a learned prediction that is then optimized via gradient ascent.

\underline{\textbf{Evaluation protocol.}} Our evaluation protocol follows prior work~\cite{brookes2019conditioning,trabucco2021designbench}:
we query each method to obtain the top $N=128$ most promising optimized samples $\x^\star_1, \cdots, \x^\star_{N}$ according to the model, and then report the $100^{\text{th}}$ percentile ground truth objective values on this set of samples, $\max(\x^\star_1, \cdots, \x^\star_N)$, as well as the $50^{\text{th}}$ percentile objective values (See Appendix~\ref{app:method_details} for numbers), averaged over 8 trials. We would argue that such an evaluation scheme is reasonable as it is typically followed in real-world MBO problems, where a set of optimized inputs are produced by the model, and the best performing one of them is finally used for deployment.

\underline{\textbf{Offline MBO tasks.}} The tasks we use can be found in the design-bench benchmark~\citep{trabucco2021designbench} at \href{https://github.com/brandontrabucco/design-bench}{github.com/brandontrabucco/design-bench}. Here we briefly summarize the tasks:   \textbf{(A)} Superconductor~\cite{fannjiang2020autofocused}, where the goal is to optimize over $86$-dimensional superconductor designs to maximize the critical temperature using 21263 points, \textbf{(B)} Hopper Controller~\cite{kumar2019model}, where the goal is to optimize over $5126$-dimensional weights of a neural network policy on the Hopper-v2 gym domain using a dataset of 3200 points, and \textbf{(C)} Ant and \textbf{(D)} D'Kitty Morphology, where the goal is to design the $60$ and $56$-dimensional morphologies, respectively, of robots to maximize policy performance using datasets, both of size 25009. {We also evaluate COMs on tasks with a discrete input space: \textbf{(E)} GFP~\citep{sarkisyan2016GFP}, where the goal is to generate the protein sequence with maximum fluorescence, \textbf{(F)} TF Bind 8, where the goal is to design a length 8 DNA sequence with high binding affinity with particular transcriptions factors and \textbf{(G)} UTR~\citep{barrera2016survey}, where the goal is to design a length 50 human 5`UTR DNA sequence with high ribosome loading. We represent discrete inputs in a transformed space of continuous-valued log probabilities for these tasks.} Results for all baseline methods are based on numbers reported by \citet{trabucco2021designbench}. Additional details for the setup of these tasks is provided in Appendix Section~\ref{appendix:task_descriptions}.

\underline{\textbf{Results on continuous tasks.}} Our results for different domains are shown in Table~\ref{tab:perf}. On three out of four continuous tasks, COMs attain the best results, in some cases (e.g. (B) HopperController) attaining the performance of over \textbf{1.3x} the best prior method.
{In addition, COMs are shown to be the only method to attain higher performance that the best training point on every task}. 
A na\"ive objective model without the conservative term, which is prone to falling off-the-manifold of valid inputs, struggles in especially high-dimensional tasks. Similarly, methods based on generative models, such as CbAS and MINs perform really poorly in the task of optimization over high-dimensional neural network weights in the HopperController task. 
These results indicate that COMs can serve as simple yet powerful method for offline MBO across a variety of domains. Furthermore, note that COMs only require training a parametric model $y = \hatf_\theta(\x)$ of the objective function with a regularizer, without any need for training a generative model, which may be harder in practice to effectively tune.

{\underline{\textbf{Results on tasks with a discrete input space.}} COMs perform competitively with the best performing methods on GFP and TF Bind8, clearly outperforming the best sample in the observed task dataset. COMs attain almost the best performance on the GFP task and outperform CbAS variants and MINs on the TF Bind8 task. In addition, COMs outperform prior methods on the {UTR} task, attaining performance within one standard deviation of the highest performing method on that task.}

{\textbf{Overall}, COMs attain the best performance on \textbf{4/7} tasks, achieving a normalized average objective value of \textbf{0.985}, improving over the next best method by \textbf{16\%} on average.
}

\begin{figure}[t]
    \centering
    \includegraphics[trim=50 5 70 5,clip,width=\linewidth]{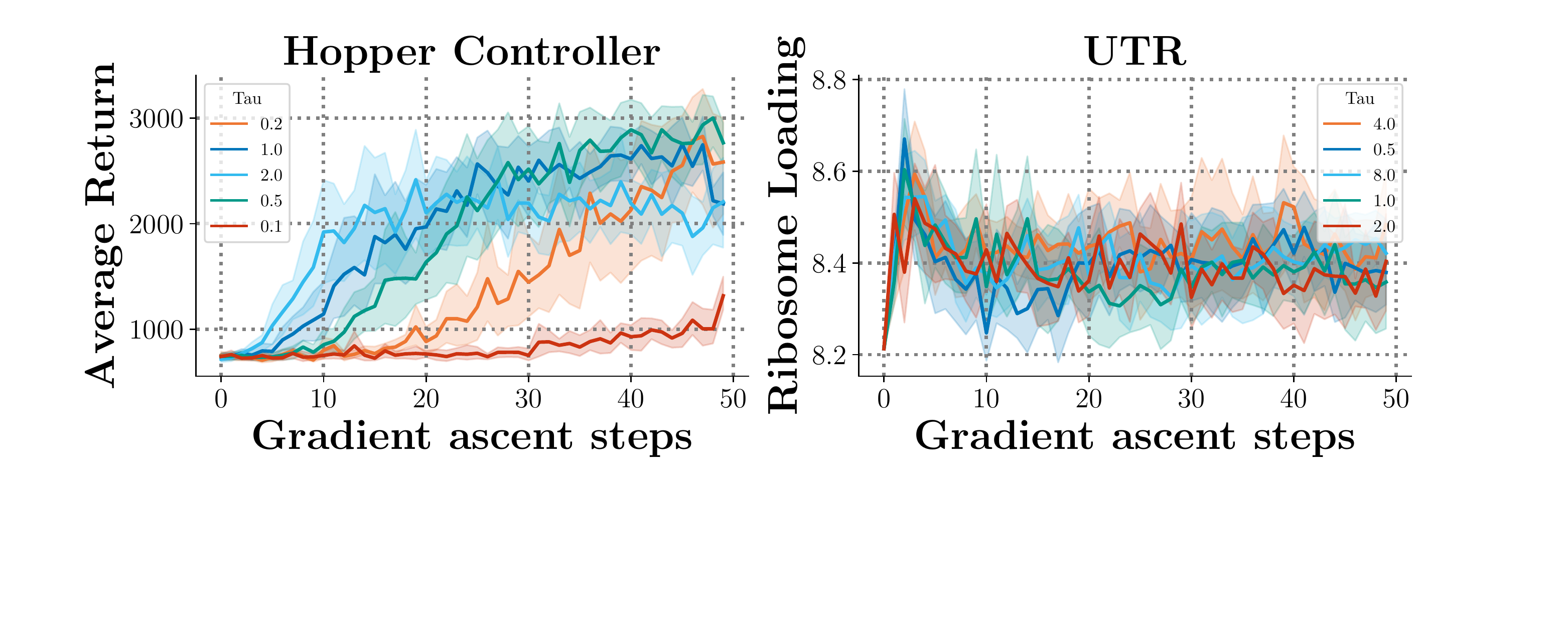}
    \vspace{-1.50cm}
    \caption{\textbf{Ablation of stability and universality of $\tau$.} In each of the two plots, we instantiate COMs on the {HopperController} and UTR tasks, and vary $\tau$ that controls the degree of conservatism (Equation~\ref{eqn:csm_training_lagrange}). The x-axis denotes the number of gradient ascent steps taken on the design $\x^{*}$ with respect to $\hatf_\theta$, and the y-axis indicates the 100th percentile of the ground truth function $\x$, which remains unobserved by the COMs algorithm, and only serves as an ablative visualization. The results demonstrate that increasing $\tau$ improves stability of COMs, and that COMs is robust to the particular choice of $\tau$. We select $\tau=0.5$ universally for continuous tasks, and $\tau=2.0$ universally for discrete tasks.}
    \label{fig:beta}
\vspace{-10pt}
\end{figure}

\subsection{Ablation Experiments}
\label{sec:ablations}

In this section, we perform an ablative experimental analysis of COMs to answer questions posed at the beginning of Section~\ref{sec:exps}. First, we evaluate the efficacy of using conservative training for learning a model of the objective function by comparing COMs to a na\"ive gradient ascent baseline and show that COMs are more \emph{stable}, i.e., the optimization performance of COMs is much less sensitive to the number of gradient ascent steps used for optimization. Second, we evaluate the effect of varying values of the Lagrange threshold $\tau$ in Equation~\ref{eqn:csm_training_lagrange}.
Third, we demonstrate the \emph{consistency} of COMs by evaluating the sensitivity of the optimization performance with respect to the number of samples $N$, that are used to compute the evaluation metric $\max(\x^*_1, \cdots, \x^*_N)$. 

\textbf{COMs are more stable than na\"ive gradient ascent.} In order to better compare COMs and a na\"ive objective model optimized using gradient ascent, we visualize the true objective value for each $\x_t$ encountered during optimization ($t$ in Line 2, Algorithm~\ref{alg:evaluation}) in Figure~\ref{fig:gradient_ascent}. Observe that a na\"ive objective model can attain good performance for a ``hand-tuned'' number of gradient ascent steps, but it soon degrades in performance  with more steps. This indicates that COMs are much more stable to the choice of number of gradient ascent steps performed than a na\"ive objective model.

\textbf{Ablation of $\tau$ in Equation~\ref{eqn:csm_training_lagrange}.} In Figure~\ref{fig:beta}, we evaluate the sensitivity of the performance of COMs as a function of the value of $\tau$. As shown in Figure~\ref{fig:beta}, we find that within the range of values evaluated, a higher value of $\tau$ gives rise to more stable optimization behavior, and we were able to utilize a universal value of $\tau=0.5$ for all tasks with a continuous input space and $\tau=2.0$ for all tasks with a discrete input space.

\begin{figure}
    \centering
    \includegraphics[trim=50 5 70 5,clip,width=\linewidth]{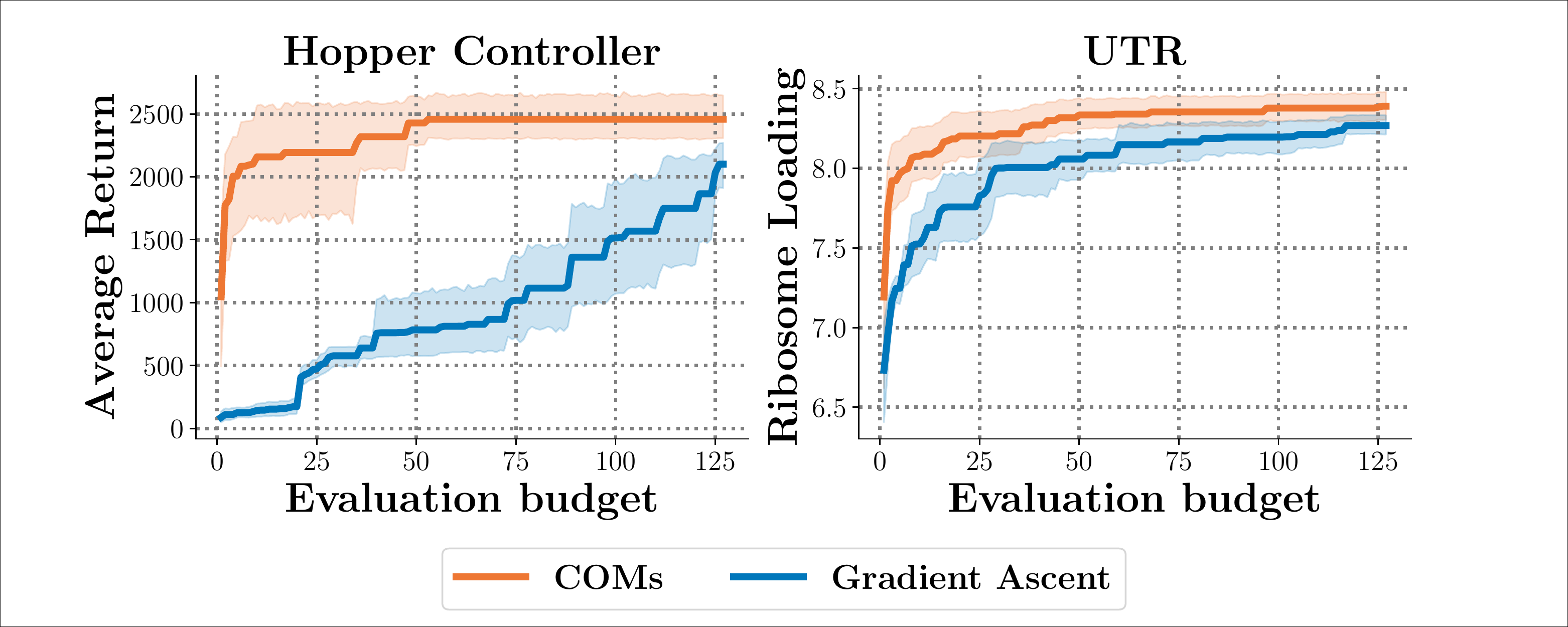}
    \vspace{-10pt}
    \caption{\textbf{Ablation of consistency of COMs by visualizing sensitivity to the post-optimization evaluation budget.} How does the performance of COMs and na\"{i}ve gradient ascent vary as the evaluation budget is reduced? In our standard evaluation, we allow each offline MBO algorithm a ``budget'' of 128 evaluations for determining 100th and 50th percentile performance.  The x-axis indicates the number $N$ of allowed evaluations, and the y-axis indicates the 100th percentile performance of the chosen $N$ points. As this evaluation budget is reduced, COMs is resilient, and remains superior to the na\"ive objective trained via supervised regression and optimized via standard gradient ascent. In the case of {HopperController}, COMs is nearly invariant to budgets down to size 55. This indicates COMs consistently produce optimized $\x^\star$ that attain high values under the true function.}
    \label{fig:budget}
    \vspace{-15pt}
\end{figure}

\textbf{COMs consistently produce well-performing inputs.} Finally, we evaluate the sensitivity of COMs to the evaluation procedure itself. Standard evaluation practice in offline MBO dictates evaluating a batch of $N$ most promising candidate inputs produced by the algorithm with the ground truth objective, where $N$ remains constant across all algorithms \cite{trabucco2021designbench,brookes2019conditioning}, and using the maximum value attained over these inputs as the performance of the algorithm, i.e., $\max(\x^*_1, \cdots, \x^*_N)$. This measures if the algorithm performs well within a provided ``evaluation budget'' of $N$ evaluations. An algorithm is more \emph{consistent} if it attains higher values of the groundtruth function with a smaller value of the evaluation budget, $N$. We used $N=128$ for evaluating all methods in Table~\ref{tab:perf}, but the value of $N$ is technically a hyperparameter and an effective offline MBO method should be resilient to this value, ideally. COMs are resilient to $N$: as we vary $N$ from 1 to 128 in Figure~\ref{fig:budget}, COMs not only perform well at larger values of $N$, but are also effective with smaller budgets, reaching near-optimal performance on HopperController in with a budget of 55, while a na\"ive objective model needs a budget twice as large to reach its own optimal performance, which is lower than that of COMs.

\section{Discussion and Conclusion}
\label{sec:discussion}
We proposed conservative objective models (COM), a simple method for offline model-based optimization, that learns a conservative estimate of the actual objective function and optimizes the input against this estimate. Empirically, we find that COMs give rise to good offline optimization performance and are considerably more stable than prior MBO methods, returning solutions that are comparable to and even better than the best existing MBO algorithms on four benchmark tasks. In this evaluation, COMs are consistently high performing, and in high-dimensional cases such as the Hopper Controller task, COMs improves on the next best method by a factor of \textbf{1.3x}.
The simplicity of COMs combined with their empirical strength make them a promising optimization backbone to find solutions to challenging and high-dimensional offline MBO problems. In contrast to certain prior methods, COMs are designed to mitigate overestimation of out-of-distribution inputs close to the input manifold, and show improved stability at good solutions. 

While our results suggest that COMs are effective on a number of MBO problems, there exists room for improvement. The somewhat na\"ive gradient-ascent optimization procedure employed by COMs can likely be improved by combining it with manifold modelling techniques, which can accelerate optimization by alleviating the need to traverse the raw input space. Similar to offline RL and supervised learning, learned objective models in MBO are prone to overfitting, especially in limited data settings. Understanding different mechanisms by which overfitting can happen and correcting for it is likely to greatly amplify the applicability of COMs to a large set of practical MBO problems that only come with small datasets.
Understanding why and how samples found by gradient ascent become off-manifold could result in a more powerful gradient-ascent optimization procedure that does not require a model-selection scheme.

\section*{Acknowledgements}
We thank anonymous ICML reviewers, Aurick Zhou and Justin Fu for discussions and feedback on the tasks and the method in this paper, and all other members from RAIL at UC Berkeley for their suggestions, feedback and support. This work was supported by National Science Foundation, the DARPA Assured Autonomy Program, C3.ai DTI, and compute support from Google, Microsoft and Intel.

\bibliography{iclr2021_conference}
\bibliographystyle{icml2019}

\clearpage
\appendix
\onecolumn
\part*{Appendices}

\section{Method Details}
\label{app:method_details}

In this section we provide additional information about our method \textbf{conservative objective models (COMs)}. In this section, we provide a 50th percentile evaluation of COMs compared to other methods and discuss additional details for COMs including including hyperparameters. Finally, we discuss how the benchmarking tasks are curated.

\subsection{Additional Results}

In addition to reporting performance using the mean 100th percentile objective value, as in Table~\ref{tab:perf}, we additionally provide a table measuring the mean 50th percentile objective value in Table~\ref{tab:perf50}. This follows the convention for evaluation standardized by \cite{trabucco2021designbench} when benchmarking model-based optimization algorithms. The 50th percentile results in Table~\ref{tab:perf50} confirm that COMs again is optimal in \textbf{4/7} tasks, the most of any method we tested, and attains a normalized average performance of \textbf{0.590}, the greatest normalized average performance of all baselines we tested.

\begin{table}[h]
    \centering
    \setlength\tabcolsep{4pt}
    \begin{tabular}{l||r|r|r|r|r}
    \toprule
    {} &            \textbf{GFP} &      \textbf{TF Bind 8} &            \textbf{UTR} & \textbf{Norm. avg. perf.} & \textbf{\# Optimal} \\
    \midrule
    $\mathcal{D}$ (\textbf{best}) & 0.789 & 0.439 & 0.593 & \\
    Auto. CbAS              &  0.848 $\pm$ 0.007 &  0.419 $\pm$ 0.007 &  0.576 $\pm$ 0.011 & 0.441 & 0 / 7 \\
    CbAS                    &  0.852 $\pm$ 0.004 &  0.428 $\pm$ 0.010 &  0.572 $\pm$ 0.023 & 0.444 & 0 / 7 \\
    BO-qEI                  &  0.246 $\pm$ 0.341 &  0.439 $\pm$ 0.000 &  0.571 $\pm$ 0.000 & 0.478 & 1 / 7 \\
    CMA-ES                  &  0.047 $\pm$ 0.000 &  0.537 $\pm$ 0.014 &  \textbf{0.612 $\pm$ 0.014} & 0.311 & 1 / 7 \\
    Grad.                   &  0.838 $\pm$ 0.004 &  0.609 $\pm$ 0.019 &  0.593 $\pm$ 0.006 & 0.464 & 1 / 7 \\
    Grad. Min               &  0.837 $\pm$ 0.001 &  \textbf{0.645 $\pm$ 0.030} &  0.598 $\pm$ 0.005 & 0.529 & 2 / 7 \\
    Grad. Mean              &  0.838 $\pm$ 0.002 &  \textbf{0.616 $\pm$ 0.023} &  \textbf{0.601 $\pm$ 0.003} & 0.528 & 3 / 7 \\
    MINs                    &  0.820 $\pm$ 0.018 &  0.421 $\pm$ 0.015 &  0.585 $\pm$ 0.007 & 0.574 & 3 / 7 \\
    REINFORCE               &  0.844 $\pm$ 0.003 &  0.462 $\pm$ 0.021 &  0.568 $\pm$ 0.017 & 0.395 & 1 / 7 \\
    \midrule
    \textbf{COMs (Ours)}    &  \textbf{0.864 $\pm$ 0.000} &  0.497 $\pm$ 0.038 &  \textbf{0.608 $\pm$ 0.012} & \textbf{0.590} & \textbf{4 / 7} \\
    \midrule
    \midrule
    {} & \textbf{Superconductor} &             \textbf{Ant Morphology} &         \textbf{D'Kitty Morphology} &          \textbf{Hopper Controller} \\
    \midrule
    $\mathcal{D}$ (\textbf{best}) & 0.399 & 0.565 & 0.884 & 1.0 \\
    Auto. CbAS              &  0.131 $\pm$ 0.010 &   0.364 $\pm$ 0.014 &  0.736 $\pm$ 0.025 &   0.019 $\pm$ 0.008 \\
    CbAS                          &  0.111 $\pm$ 0.017 &   0.384 $\pm$ 0.016 &  0.753 $\pm$ 0.008 &   0.015 $\pm$ 0.002 \\
    BO-qEI                        &  0.300 $\pm$ 0.015 &   0.567 $\pm$ 0.000 &  \textbf{0.883 $\pm$ 0.000} &   0.343 $\pm$ 0.010 \\
    CMA-ES                        &  0.379 $\pm$ 0.003 &  -0.045 $\pm$ 0.004 &  0.684 $\pm$ 0.016 &  -0.033 $\pm$ 0.005 \\
    Grad.               &  \textbf{0.476 $\pm$ 0.022} &   0.134 $\pm$ 0.018 &  0.509 $\pm$ 0.200 &   0.092 $\pm$ 0.084 \\
    Grad. Min  &  \textbf{0.471 $\pm$ 0.016} &   0.185 $\pm$ 0.008 &  0.746 $\pm$ 0.034 &   0.222 $\pm$ 0.065 \\
    Grad. Mean &  \textbf{0.469 $\pm$ 0.022} &   0.187 $\pm$ 0.009 &  0.748 $\pm$ 0.024 &   0.243 $\pm$ 0.064 \\
    MINs                          &  0.336 $\pm$ 0.016 &   \textbf{0.618 $\pm$ 0.040} &  \textbf{0.887 $\pm$ 0.004} &   \textbf{0.352 $\pm$ 0.058} \\
    REINFORCE                     &  \textbf{0.463 $\pm$ 0.016} &   0.138 $\pm$ 0.032 &  0.356 $\pm$ 0.131 &  -0.064 $\pm$ 0.003 \\
    \midrule
    \textbf{COMs (Ours)}          &  0.386 $\pm$ 0.018 &   0.519 $\pm$ 0.026 &  \textbf{0.885 $\pm$ 0.003} &   \textbf{0.375 $\pm$ 0.003} \\
    \bottomrule
    \end{tabular}
    \caption{\textbf{Comparative evaluation of COMs} against prior methods in terms of the mean 50th-percentile score and its standard deviation over 8 trials. Tasks include {Superconductor}, {HopperController}, {AntMorphology}, and {DKittyMorphology}, which have a continuous design input space and {GFP}, {TFBind8} and {UTR} with a discrete design input space.} 
    \label{tab:perf50}
\end{table}

\subsection{Implementation details}
In addition to various considerations from Sections~\ref{sec:method_addl} and \ref{sec:method_overall},  one important implementation detail of COMs is to normalize the inputs $(\x)$ and outputs ($y$ values) for training the conservative model, $\hatf_\theta(\x)$. Our motivation for using normalization was simple: Since the input and output ranges and modalities of various tasks we evaluated on in Table~\ref{tab:perf} is very different from each other, in order to be able to use a \textit{uniform} set of hyperparameters for COMs, it is necessary to normalize both the inputs $\x$ and outputs $y$ to a standard range. Following standard normalization practices, we normalized $\x$ and $y$ such that the resulting first and second moments match those of a unit Gaussian distribution. In practice, this means collecting all objective values from the training dataset into a vector $Y \in \mathbb{R}^{N \times 1}$, evaluating the sample mean $\hat \mu = \text{mean}(Y)$ and sample standard deviation $\hat \sigma = \text{std}(Y - \hat \mu)$. A similar procedure is used for calculating the sample mean and sample standard deviation of $\x$. The objective values and inputs are then normalized by subtracting their sample mean and dividing by their sample standard deviation $y \leftarrow (y - \hat \mu) / \hat \sigma$, except where doing so would divide by zero. This normalization allows COMs to use the uniform set of hyperparameters, which we mention explicitly, in Table~\ref{tab:hyperparam}.

\begin{table}[]
    \centering
    \begin{tabular}{l|r|r}
        \textbf{Hyperparameter} & \textbf{Discrete} & \textbf{Continuous} \\
        \midrule
        Number of epochs to train $\hatf_\theta$ & 50 & 50 \\
        $T$ (Number of gradient ascent steps using Equation~\ref{eqn:csm_optimization_constraint}) & 50 & 50 \\
        Number of steps used to generate adversarial samples $\mu(\x)$ in Equation~\ref{eqn:csm_training_lagrange} & 50 & 50 \\
        $\alpha$ learning rate (used to optimize Equation~\ref{eqn:csm_training_lagrange} via dual gradient descent) & 0.01 & 0.01 \\
        $\tau$ in Equation~\ref{eqn:csm_training_lagrange} & 2.0 & 0.5  \\
        $\eta$ in Equation~\ref{eqn:csm_optimization_constraint} & $2.0 \sqrt{d}$ & $0.05 \sqrt{d}$
    \end{tabular}
    \caption{\textbf{Hyperparameters for COMs.} All hyperparameters are kept constant across all discrete tasks and continuous tasks respectively in COMs. The variable $d$ indicates the cardinality of a single design $\x$ in the training set of the model. Scaling the learning rate by a factor proportional to $\sqrt{d}$ follows the implementation of the Gradient ascent baseline from \citet{trabucco2021designbench}}
    \label{tab:hyperparam}
\end{table}

\subsection{Benchmarking Details}

In order to promote reproducibility, we additionally provide the task identifiers and keyword arguments used with the \textit{design-bench} \citet{trabucco2021designbench} package. These arguments are passed to the \texttt{design\_bench.make} function call in order to build a model-based optimization Task object in Python. Note that in addition to specifying the name of the task dataset (such as GFP), one must also specify the desired oracle function (such as a Transformer). In Table~\ref{tab:taskid} we detail the specific combination of task datasets and oracle functions used in this work. Additionally, when an approximate oracle is used, commonly because an exact simulator or closed form equation for the ground truth y values is not available, there is a train-test discrepancy, where the predictions of the approximate oracle may not be a perfect reflection of the ground truth y values contained in the original model-based optimization dataset. This discrepancy is further explored by \citet{trabucco2021designbench}; however, we find that UTR is particular susceptible to such discrepency, and so we choose to relabel the y values contained in the MBO dataset with the predictions of the CNN oracle. See Appendix~\ref{appendix:task_descriptions} for more information.

\begin{table}[h]
    \centering
    \begin{tabular}{l|r|r}
        \textbf{Task} & \textbf{Design-Bench ID} & \textbf{Relabel} \\
        \midrule
        GFP & \texttt{GFP-Transformer-v0} & False \\
        TF Bind 8 & \texttt{TFBind8-Exact-v0} & False \\
        UTR & \texttt{UTR-ResNet-v0} & True \\
        Superconductor & \texttt{Superconductor-RandomForest-v0} & False \\
        Ant Morphology & \texttt{AntMorphology-Exact-v0} & False \\
        D'Kitty Morphology & \texttt{DKittyMorphology-Exact-v0} & False \\
        Hopper Controller & \texttt{HopperController-Exact-v0} & False \\
    \end{tabular}
    \caption{\textbf{Design-Bench task identifiers.} This table contains the necessary arguments to pass to the \texttt{design\_bench.make} function call. More information is available at \url{https://github.com/brandontrabucco/design-baselines}, and documentation for Design-Bench is available at \url{https://github.com/brandontrabucco/design-bench}}
    \label{tab:taskid}
\end{table}

\section{Proof of Theorem~\ref{thm:lower_bound}}
\label{appendix:theorem_proof}

In this section, we provide a proof for Theorem~\ref{thm:lower_bound} and show that the conservative training by performing gradient descent on $\theta$ with respect to the objective in Equation~\ref{eqn:csm_analysis} (restated below in a more convenient form as Equation~\ref{eqn:csm_analysis_again}) indeed obtains a conservative model of the actual objective function. Note that $\overline{\mathcal{D}}(\x'|\x)$ denotes a smoothed Dirac-delta distribution centered at $\x$, which can be obtained by adding random noise to a given $\x$.
\begin{equation}
    \label{eqn:csm_analysis_again}
\mathcal{L}(\theta; \mu, \overline{\mathcal{D}}) := \alpha \left( \expec_{\textcolor{red}{\x_0 \sim \overline{\mathcal{D}}, \x_T \sim \mu(\x_T|\x_0)}}\left[\hatf_\theta(\x_T)\right] - \expec_{\x \sim \overline{\mathcal{D}}, \x_T \sim \overline{\mathcal{D}}(\x_T|\x_0)}\left[\hatf_\theta(\x_T)\right] \right) + \underbrace{\frac{1}{2} \mathbb{E}_{\x_0 \sim \overline{\mathcal{D}}, (\x, y) \sim \overline{\mathcal{D}}(\x|\x_0)}\left[ \left(\hatf_\theta(\x) - y \right)^2 \right]}_{:= (\wedge)}.
\end{equation}

We now restate a formal version of Theorem~\ref{thm:lower_bound} (including correcting a typo from the informal version in the main paper, where the argument of the left hand side of the expression was denoted as $\x'$ instead of $\x''$), and then provide a proof. We make an additional assumption that the neural tangent kernel, $\bG^k_f(\x, \x')$,  is semi-positive definite.  
\begin{theorem}[Formal version of Theorem~\ref{thm:lower_bound}]
\label{thm:lower_bound_restated}
Assume that $\hat{f}_\theta(\x)$ is trained by performing gradient descent on $\theta$ with respect to the objective $\mathcal{L}(\theta; \mu,\overline{\mathcal{D}})$ in Equation~\ref{eqn:csm_analysis_again} with a learning rate $\eta$. The parameters in step $k$ of gradient descent are denoted by $\theta^k$, and let the corresponding conservative model be denoted as $\hat{f}^k_\theta$. Let $\bG$, $\mu$, $\lhat$, $L$, $\overline{\data}$ be defined as discussed above. Then, under assumptions listed above, $\forall~ \x \in \data, \x'' \in \mathcal{X}$, the conservative model at iteration $k+1$ of training satisfies:
\begin{align*}
    \hat{f}_\theta^{k+1}(\x'') := \max~ \Big\{\hat{{f}}^{k+1}_\theta(\x) - \lhat ||\x'' - \x||_2,
    \tilde{f}_\theta^{k+1}(\x'') - \eta \alpha \E_{\x \sim \overline{\data}, \x' \sim \mu}[ \bG^k_f(\x'', \x')] + \eta \alpha \E_{\x \sim \overline{\data}, \x' \sim \overline{\data}} [\bG^k_f(\x'', \x')]\Big\},
\end{align*}
where $\tilde{f}_\theta^{k+1}(\x'')$ is the resulting $(k\!+\!1)$-th iterate of $\hat{f}_\theta$ if conservative training were not used. Thus, if $\alpha$ is sufficiently large, the expected value of the asymptotic function, $\hat{f}_\theta := \lim_{k \rightarrow \infty} \hat{f}_\theta^k$, on inputs $\x_T$ found by the optimizer, lower-bounds the value of the true function $f(\x_T)$: $$\E_{\x_0 \sim \data, \x_T \sim \mu(\x_T|\x_0)}[\hat{f}_\theta(\x_T)] \leq \E_{\x_0 \sim \data, \x_T \sim \mu(\x_T|\x_0)}[f(\x)].$$  
\end{theorem}
\begin{proof}
For proving the first part of the theorem, we first derive the expression for the gradient of $\mathcal{L}(\theta; \mu, \overline{\mathcal{D}})$ with respect to $\theta$, and denote the y-value for a given $\x$ as a deterministic function $y(\x)$. Our proof can directly be extended to a non-deterministic $y(\x)$ with an additional integral over $y$ values, but we stick to deterministic $y(\x)$ for simplicity.
\begin{equation*}
    \nabla_\theta \mathcal{L}(\theta; \mu, \overline{\mathcal{D}}) = \alpha \int \left( \overline{\mathcal{D}}(\x_0) \mu(\x|\x_0) - \overline{\mathcal{D}}(\x_0) \overline{\mathcal{D}}(\x|\x_0) \right) \nabla_\theta \hatf_\theta(\x) ~d\x_0 d \x +  \int \overline{\mathcal{D}}(\x_0) \overline{\mathcal{D}}(\x|\x_0) (f_\theta(\x) - y(\x)) \nabla_\theta \hatf_\theta(\x)~ d\x d\x_0.
\end{equation*}
At any iteration $k$ of gradient descent, the next parameter iterate $\theta^{k+1}$ are obtained via, $\theta^{k+1} = \theta^k - \eta \nabla_\theta \mathcal{L}(\theta; \mu, \overline{\mathcal{D}})$. Using this relation, and making an approximate linearization assumption on the non-linear function $\hatf^{k}_\theta$ for a small learning rate $\eta << 1$ under the assumption of the neural tangent kernel (NTK)~\citep{jacot2018neural} regime, which models the behavior of deep neural networks in the infinite-width limit, we obtain the expression for the next function value: $\hat{f}^{k+1}_\theta(\x'')$:
\begin{align*}
    &\hatf^{k+1}_\theta(\x'') \approx \hat{f}^k_\theta(\x'') + (\theta^{k+1} - \theta^k)^T \nabla_\theta \hat{f}^k_\theta (\x'')\\
    &= \underbrace{\hat{f}^k_\theta(\x'') + \eta \E_{\x \sim \overline{\data}, \x' \sim \overline{\data}}[\left(y(\x') - \hat{f}^k_\theta(\x')\right) \bG^k_f(\x'', \x')]}_{:= (*)} - \underbrace{\left(\eta \alpha \E_{\x \sim \overline{\data}, \x' \sim \mu}[\bG^k_f(\x'', \x')] - \eta \alpha \E_{\x \sim \overline{\data}, \x' \sim \overline{\mathcal{D}}} [\bG^k_f(\x'', \x')]\right)}_{:= \Delta(\x'')},
\end{align*}
where the expression marked as $(*)$ denotes the $(k+1)$-th iterate of the function, under gradient descent on just the mean-squared error $(f(\x) - y)^2$ term, marked as $(\wedge)$ in Equation~\ref{eqn:csm_analysis_again}. Noting that the theorem statement denotes the term $(*)$ as $\tilde{f}_\theta^{k+1}(\x'')$, we obtain our first desired result. To obtain the first argument of the $\max$ in the theorem statement, note that if the function $\hatf^{k+1}_\theta$ is $\lhat$-Lipschitz, the value at $\x''$ cannot be smaller than $\hat{f}^{k+1}_\theta(\x) - \lhat||\x -\x''||_2$, and hence the maximum over the two terms. 

For proving the second part of the theorem statement, observe that if we can show that in expectation over $\x'' \sim \mu(\x_T); \mu(\x_T) := \int_{\x_0} \overline{\mathcal{D}}(\x_0) \mu(\x_T|\x_0)~ d\x_0$, the quantity $\Delta(\x'')$ is positive, then our argument is complete since we have shown that each step of gradient descent on $\theta$ reduces the value of $\E_{\x_0 \sim \overline{\mathcal{D}}, \x_T \sim \mu(\x_T|\x_0)}[\hatf^k_\theta(\x_T)]$ by a positive quantity by virtue of training with Equation~\ref{eqn:csm_analysis_again} as compared to only training $\theta$ with standard squared error $(\wedge)$. Thus, if $\E_{\x_0 \sim \overline{\mathcal{D}}, \x_T \sim \mu(\x_T|\x_0)}[\Delta^k(\x_T)]$ is positive for all gradient descent steps $k$, we obtain the desired lower-bound condition as $k \rightarrow \infty$. As an additional detail, note that we assumed $\lhat >>  L$ (i.e. the Lipschitz constant of $\hatf_\theta(\x)$ is sufficiently larger than that of $f(\x)$). This condition handles the boundary case when the predictions $\hatf^{k+1}_\theta(\x')$ get lower-bounded under the first argument of $\max$ in the first part of Theorem~\ref{thm:lower_bound_restated} due to the Lipschitz condition: $\hat{{f}}^{k+1}_\theta(\x) - \lhat ||\x' - \x||_2$. 

Finally, we fill in the missing piece that show $\E_{\x_0 \sim \overline{\mathcal{D}}, \x_T \sim \mu(\x_T|\x_0)}[\Delta^k(\x_T)]$ is positive for each $k$. Under the assumption that the neural tangent kernel $\bG^{k}(\x, \x')$ is semi-positive definite for all $k$, we can express:
\begin{align*}
    \E_{\x_0 \sim \overline{\mathcal{D}}, \x_T \sim \mu(\x_T|\x_0)}[\Delta^k(\x_T)] &:= \eta \alpha \int_{\x, \x_0, \x', \x_T} \left[ \overline{\data}(\x) \mu(\x'|\x) - \overline{\data}(\x) \overline{\data}(\x'|\x) \right] \overline{\data}(\x_0) \mu(\x_T|\x_0) \bG^k_f(\x', \x_T)\\
    &= \eta \alpha \int_{\x_0} \overline{\data}(\x_0) \int_{\x} \overline{\data}(\x) \int_{\x', \x_T} [\mu(\x'|\x) - \overline{\data}(\x'|\x)] \mu(\x_T|\x_0) \bG^k_f(\x', \x_T)
\end{align*}
By now writing the above in matrix form, we note that the RHS of the above equation has the same structure as the second term in the RHS of Equation 14 in \citet{kumar2020conservative}, and furthermore since $\bG^k_f$ is positive semi-definite, it satisfies the required conditions for Equation 14 and Theorem D.1 from \citet{kumar2020conservative} to be applicable. Thus, exactly following the proof of Theorem D.1 in \citet{kumar2020conservative} for the linear function approximation case in reinforcement learning, with the following substitutions: $P_\mathbf{F} := \bG^k_f(\cdot, \x_T)$ (i.e., a column of the kernel Gram-matrix for a fixed value of the second argument) and $a = \x_T$, $s = \x_0$, we can show that $\E_{\x_0 \sim \overline{\mathcal{D}}, \x_T \sim \mu(\x_T|\x_0)}[\Delta^k(\x_T)] \geq 0$, thus finishing our argument.  
\end{proof}

\section{Network Details}
\label{appendix:network}

In each of our experiments, we train a neural network $\hatf_\theta$ to approximate the ground truth score function of an offline MBO task, where $\theta$ represents the weights of the model. Distinct from prior methods based on generative models~\citep{kumar2019model,brookes2019conditioning} we are able to utilize the same neural network architecture for representing the learned model, $\hatf_\theta(\x)$ across all MBO tasks. This architecture is a three-layer neural network with two hidden layers of size 2048, followed by Leaky ReLU activation functions with a leak of $0.3$. Each neural network $\hatf_\theta$ has an output layer that predicts a single scalar objective value $y$, which is used for regression. Specifically, $\hatf_\theta$ is trained to minimize the mean squared error of observed objective values, using the default parameters of the Adam optimizer as discussed in Section~\ref{sec:method_overall}.

\section{Data Collection}
\label{appendix:task_descriptions}

In this section, we detail the data collection steps used for creating each of the tasks from \cite{trabucco2021designbench}, used for benchmarking COMs. We answer \textbf{(1)} where is the data from, and \textbf{(2)} what pre-processing steps are used?

\subsection{TF Bind 8}

The TF Bind 8 task is a derivative of the transcription factor binding activity survey performed by \citet{barrera2016survey}, where the binding activity scores of every possible length eight DNA sequence was measured with a variety of human transcription factors. We filter the dataset by selecting a particular transcription factor \texttt{SIX6\_REF\_R1}, and defining an optimization problem where the goal is to synthesize a length 8 DNA sequence with high binding activity with human transcription factor \texttt{SIX6\_REF\_R1}. This particular transcription factor for TF Bind 8 was recently used for optimization in \citet{Angermueller2020Model-based, angermueller2020population}. TF Bind 8 is a fully characterized dataset containing 65792 samples, representing every possible length 8 combination of nucleotides $\bx_{\text{TFBind8}}\in\{0,1\}^{8\times4}$. The training set given to offline MBO algorithms is restricted to the bottom 50\%, which results in a visible training set of 32898 samples.

\subsection{GFP}

The GFP task provided is a derivative of the GFP dataset \cite{sarkisyan2016GFP}. The dataset we use in practice is that provided by \citet{brookes2019conditioning} at the url \url{https://github.com/dhbrookes/CbAS/tree/master/data}. We process the dataset such that a single training example consists of a protein represented as a tensor $\bx_{\text{GFP}}\in\{0,1\}^{237\times20}$. This tensor is a sequence of 237 one-hot vectors corresponding to which amino acid is present in that location in the protein. We use the dataset format of \cite{brookes2019conditioning} with no additional processing. The data was originally collected by performing laboratory experiments constructing proteins similar to the Aequorea victoria green fluorescent protein and measuring fluorescence. We employ the full dataset of 56086 proteins when learning approximate oracles for evaluating offline MBO methods, but restrict the training set given to offline MBO algorithms to 5000 samples drawn from between the 50th percentile and 60th percentile of proteins in the GFP dataset, sorted by fluorescence values. This subsampling procedure is consistent with the procedure used by prior work \cite{brookes2019conditioning}.

\subsection{UTR}

The UTR task is derived from work by \citet{sample2019human} who trained a CNN model to predict the expressive level of a particular gene from a corresponding 5`UTR sequence. Our use of the UTR task for model-based optimization follows \citet{angermueller2020population}, where the goal is to design a length 50 DNA sequence to maximize expression level. We follow the methodology set by \citet{sample2019human} to sort all length 50 DNA sequences in the unprocessed UTR dataset by total reads, and then select the top 280,000 DNA sequences with the most total reads. The result is a dataset containing 280,000 samples of length 50 DNA sequences $\bx_{\text{UTR}}\in\{0,1\}^{50\times4}$ and corresponding ribosome loads. When training offline MBO algorithms, we subsequently eliminate the top 50\% of sequences ranked by their ribosome load, resulting in a visible dataset with only 140,000 samples.  

\subsection{Superconductor}

The Superconductor task is inspired by recent work \cite{fannjiang2020autofocused} that applies offline MBO to optimize the properties of superconducting materials for high critical temperature. The data we provide in our benchmark is real-world superconductivity data originally collected by \cite{hamidieh2018superconductor}, and subsequently made available to the public at \url{https://archive.ics.uci.edu/ml/datasets/Superconductivty+Data#}. The original dataset consists of superconductors featurized into vectors containing measured physically properties like the number of chemical elements present, or the mean atomic mass of such elements. One issue with the original dataset that was used in \cite{fannjiang2020autofocused} is that the numerical representation of the superconducting materials did not lend itself to recovering a physically realizable material that could be synthesized in a lab after performing model-based optimization. In order to create an \textit{invertible} input specification, we deviate from prior work and encode superconductors as vectors whose components represent the number of atoms of specific chemical elements present in the superconducting material---a serialization of the chemical formula of each superconductor. The result is a real-valued design space with 86 components $\bx_{\text{Superconductor}}\in\mathcal{R}^{86}$. The full dataset used to learn approximate oracles for evaluating MBO methods has 21263 samples, but we restrict this number to 17010 (the 80th percentile) for the training set of offline MBO methods to increase difficulty. 

\subsection{Hopper Controller}

The goal of the Hopper Controller task is to design a set of weights for a neural network policy that achieves high expected return. The data collected for HopperController was taken by training a three layer neural network policy with 64 hidden units and 5126 total weights on the Hopper-v2 MuJoCo task using Proximal Policy Optimization \cite{Schulman2017ppo}. Specifically, we use the default parameters for PPO provided in stable baselines \cite{stable-baselines}. The dataset we provide with this benchmark has 3200 unique weights. In order to collect this many, we run 32 experimental trials of PPO, where we train for one million steps, and save the weights of the policy every 10,000 environment steps. The policy weights are represented originally as a list of tensors. We first traverse this list and flatten each of the tensors, and we then concatenate each of these flattened tensors into a single training example $\bx_{\text{Hopper}}\in\mathcal{R}^{5126}$. The result is an optimization problem over neural network weights. After collecting these weights, we perform no additional pre-processing steps. In order to collect scores we perform a single rollout for each $x$ using the Hopper-v2 MuJoCo environment. The horizon length for training and evaluation is limited to 1000 simulation time steps, which is standard practice for this MuJoCo environment.

\subsection{Ant \& D'Kitty Morphology}

Both morphology tasks share methodology. The goal of these tasks is to design the morphology of a quadrupedal robot---an ant or a D'Kitty---such that the agent is able to crawl quickly in a particular direction. In order to collect data for this environment, we create variants of the MuJoCo Ant and the ROBEL D'Kitty agents that have parametric morphologies. The goal is to determine a mapping from the morphology of the agent to the average return of a pre-trained morphology conditioned agent. We implement this by pre-training a morphology conditioned neural network policy using SAC \cite{Haarnoja2018SoftAO}. For both the Ant and the D'Kitty, we train the agents for more than ten million environment steps, and a maximum episode length of 200, with all other settings as default. These agents are pre-trained on Gaussian distributions of morphologies. The Gaussian distributions are obtained by adding Gaussian noise with standard deviation 0.03 for Ant and 0.01 for D'Kitty the design-space range to the default morphologies. 

After obtaining trained morphology-conditioned policies, we create a dataset of morphologies for model-based optimization by sampling initialization points randomly, and then using CMA-ES to optimize for morphologies that attain high reward using the pretrained morphology-conditioned policy. To obtain initialization points, we add Gaussian random noise to the default morphology for the Ant with standard deviation 0.075 and D'Kitty with standard deviation 0.1, and then apply CMA-ES with standard deviation 0.02. We ran CMA-ES for 250 iterations and then restart, until a minimum of 25000 morphologies were collected, resulting in a final dataset size of 25009 for both the Ant and D'Kitty. The design space for Ant Morphologies is $\bx_{\text{Ant}}\in\mathcal{R}^{60}$, whereas for D'Kitty morphologies is $\bx_{\text{D'Kitty}}\in\mathcal{R}^{56}$. We subsample the dataset to its 40th percentile when training offline MBO algorithms, resulting in 10004 samples.

\section{Oracle Functions}

We detail oracle functions for evaluating ground truth scores for each task. A common thread is that the oracle, if trained, is fit to a larger static dataset containing higher performing designs than observed by a downstream MBO algorithm.

\subsection{TF Bind 8}

TF Bind 8 is a fully characterized discrete offline MBO task, which means that all possible designs have been evaluated \cite{barrera2016survey} and are contained in the full hidden TF Bind 8 dataset. The oracle for TF Bind 8 is therefore implemented as a lookup table that returns the score corresponding to a particular length 8 DNA sequence from the dataset. By restricting the size of the training set visible to an offline MBO algorithm, it is possible for the algorithm to propose a design that achieves a higher score than any other DNA sequence visible to the offline MBO algorithm during training.

\subsection{GFP}

GFP uses a simplified Transformer to the TAPE Transformer proposed by \citet{RaoBTDCCAS19}. The Transformer used has 4 attention blocks and a hidden size of 64. The Transformer is fit to the entire hidden GFP dataset, making it possible to sample a protein design that achieves a higher score than any other protein visible to an offline MBO algorithm. The model has a Spearman's rank correlation coefficient of 0.8497 with a held-out validation set derived from the GFP dataset.

\subsection{UTR}

UTR uses a CNN, which differs from the CNN that was originally used by \cite{angermueller2020population} in that it has residual connections. Our reasoning for making this change is that ResNet is a newer and possibly higher capacity model that may be less prone to mistakes than the shallower CNN model proposed by \citet{sample2019human}. The chosen CNN has 2 residual blocks with 2 convolution layer each, and a hidden size of 120. The CNN is fit to the entire hidden UTR dataset, making it possible to sample a DNA sequence that achieves a higher score than any other sequence visible to an offline MBO algorithm. The resulting CNN has a spearman's rank correlation coefficient of 0.8617 with a held-out validation set.

\subsection{Superconductor}

The Superconductor oracle function is also a random forest regression model. The model we use it the model described by \cite{hamidieh2018superconductor}. We borrow the hyperparameters described by them, and we use the RandomForestRegressor provided in scikit-learn. Similar to the setup for the previous set of tasks, this oracle is trained on the entire hidden dataset of superconductors. The random forest has a spearman's rank correlation coefficient with a held-out validation set of 0.9155.

\subsection{HopperController}

HopperController and the remaining tasks implement an exact oracle function. For HopperController the oracle takes the form of a single rollout using the Hopper-v2 MuJoCo environment. The designs for HopperController are neural network weights, and during evaluation, a policy with those weights is instantiated---in this case that policy is a three layer neural network with 11 input units, two layers with 64 hidden units, and a final layer with 3 output units. The intermediate activations between layers are hyperbolic tangents. After building a policy, the Hopper-v2 environment is reset and the reward for 1000 time-steps is summed. That summed reward constitutes the score returned by the Hopper Controller oracle. The limit of performance is the maximum return that an agent can achieve in Hopper-v2 over 1000 steps. 

\subsection{Ant \& D'Kitty Morphology}

The final two tasks in design-bench use an exact oracle function, using the MuJoCo simulator. For both morphology tasks, the simulator performs a rollout and returns the sum of rewards at every timestep in that rollout. Each task is accompanied by a pre-trained morphology-conditioned policy. To perform evaluation, a morphology is passed to the Ant or D'Kitty MuJoCo environments respectively, and a dynamic-morphology agent is initialized inside these environments. These simulations can be time consuming to run, and so we limit the rollout length to 100 steps. The morphology conditioned policies were trained using Soft Actor Critic for 10 million steps for each task, and are ReLU networks with two hidden layers of size 64.

\end{document}